\documentclass{article}

\usepackage{arxiv}

\usepackage[utf8]{inputenc}
\usepackage[T1]{fontenc}
\usepackage{hyperref}
\usepackage{url}
\usepackage{booktabs}
\usepackage{amsfonts}
\usepackage{nicefrac}
\usepackage{microtype}
\usepackage{graphicx}
\usepackage{amsmath,amssymb}
\usepackage{bm}
\usepackage{mathrsfs}
\usepackage{amsthm}
\usepackage{multirow}
\usepackage{float}
\usepackage{xcolor}

\newtheorem{proposition}{Proposition}

\title{Symmetry-Structured Neural Completion of Islamic Geometric Patterns from Sparse Control Geometry}

\author{
  Hassan Ugail \\
  Centre for Visual Computing and Intelligent Systems \\
  University of Bradford \\
  United Kingdom \\
  \And
  Irfan Mehmood \\
  Centre for Visual Computing and Intelligent Systems \\
  University of Bradford \\
  United Kingdom \\
}

\begin{document}
\maketitle

\begin{abstract}
Islamic geometric patterns are governed by exact rotational symmetry and strict construction rules. This paper treats these rules as formal geometric knowledge and embeds them in a neural completion framework, rather than leaving them to be learned statistically from data. Given sparse control geometry and a target symmetry order, the system completes the pattern as a vector graph by predicting edges and refinements of bounded curves over a candidate lattice whose edges are organised into rotational orbits under the cyclic group. Symmetry is enforced either by constraining predictions within these orbits or by projecting them onto them during inference. The orbit-tied variant provides a constructive guarantee: for any input and any orbit-level selection rule, it produces exact $N$-fold symmetry, preserves anchor points, and keeps all refinements within prescribed bounds. These properties are verified numerically. The study focuses on rotational symmetry, and all quantitative results are obtained from procedurally generated graphs inspired by Islamic geometric design rather than from a historical corpus. On clean inputs, enforcing exact validity produces no measurable loss in fidelity. When control geometry is missing, an unstructured decoder loses fidelity and breaks symmetry; retraining on corrupted inputs recovers much of the fidelity but not exact validity. Symmetry-structured inference, by contrast, keeps violations at zero throughout. The results show that augmentation and symmetry structure address distinct failure modes: augmentation improves fidelity under corruption, while symmetry structure guarantees validity. The framework therefore provides a knowledge-constrained, guarantee-backed approach to neural completion for scalable vector ornaments whose validity depends on exact geometric structure.
\end{abstract}


\section{Introduction}\label{sec:intro}

Islamic geometric ornament is defined as much by its geometry as by its visual appearance. Star patterns, rosettes, and related designs are constructed from rotational symmetry, repeated construction elements, fixed anchor points, and coherent vector geometry. Historically, such patterns were laid out using compass-and-straightedge methods rather than drawn freehand \cite{bonner2017,broug2019,necipoglu1996}. Their mathematical structure has been studied extensively, from the symmetry-group analysis of Abas and Salman \cite{abas1995} to the decagonal and quasi-crystalline girih structures found in medieval architecture \cite{lu2007}. In this setting, even a small symmetry violation is not simply a visual imperfection; it can make the result invalid as an instance of the tradition.

This makes Islamic geometric ornament a demanding test case for neural generation. The present study examines this problem using procedurally generated graphs inspired by the tradition, rather than a historical corpus. Modern neural image models can produce visually convincing patterns, but they typically learn symmetry as a statistical regularity and cannot certify that it holds exactly. We therefore frame the problem as completion rather than image synthesis. The input is sparse control geometry: a small set of construction points and anchors, together with a target symmetry order $N$. From this input, the system completes the full pattern as a vector graph by selecting edges from a geometrically meaningful candidate lattice and adding bounded-curve refinements where appropriate. The output is an exactly $N$-fold symmetric ornament that can be exported as scalable vector graphics.

From a knowledge-driven perspective, the construction rules are the domain knowledge. They determine which vertices are admissible, which edges may be selected, how rotational orbits are organised, where anchors must remain fixed, and what counts as a valid completion. The learning problem is therefore not only to infer plausible missing edges. It is to combine statistical prediction with explicit geometric constraints so that every decoded output satisfies the rules of the construction.

Predicting plausible edges is comparatively straightforward: a standard graph network can achieve competitive edge-level accuracy. The harder problem is to ensure that the predicted graph also preserves orbit consistency, exact rotational symmetry, anchor positions, and bounded refinements, especially when the input geometry is incomplete or degraded. Our experiments expose this distinction. On clean inputs, an unstructured decoder matches the structured methods in edge-level fidelity, but still breaks symmetry. When part of the control geometry is removed, the clean-trained unstructured model loses substantial fidelity. Retraining on corrupted inputs recovers much of this fidelity, but symmetry violations remain. By contrast, symmetry-structured inference remains exactly valid in every tested condition, whether the structure is enforced within the architecture or applied by projection at inference.

The paper makes four contributions. First, it introduces a symmetry-structured graph neural completion framework for Islamic geometric patterns. In this framework, constructional knowledge defines the candidate lattice, the orbit structure, the admissible edges, and the validity-preserving decoder, while a graph network predicts edge scores and bounded curve refinements over a lattice whose edges are grouped into rotational orbits under $C_N$.

Second, it establishes validity by construction. The embedded geometry enforces exact $N$-fold symmetry, fixed anchors, and bounded refinement for any input and any orbit-level selection rule. These properties are verified numerically for every evaluated pattern.

Third, it provides a controlled comparison of four decoding regimes: orbit-tied prediction, orbit-level selection projection, full projection of scores and refinements, and unstructured free decoding. The main finding is not that orbit tying is more accurate than projection. Rather, both forms of symmetry-structured inference preserve validity with no measurable cost in fidelity, showing that the essential ingredient is the orbit structure of the candidate geometry.

Fourth, it combines a robustness study with a critique of evaluation in this domain. When control geometry is removed, the unstructured model loses fidelity and accumulates symmetry violations, while the structured methods remain valid. At the same time, a procedural template achieves a high raw overlap score by selecting too many edges, demonstrating that overlap alone is insufficient for evaluating geometric completion.

Overall, the results show that exact geometric validity can be enforced without sacrificing fidelity. They also show that when the control geometry is incomplete, symmetry-structured inference returns valid completions that unstructured decoding does not. This is the paper's central contribution to knowledge-driven artificial intelligence: it demonstrates how explicit geometric knowledge can constrain a neural predictor so that its outputs are not merely plausible, but certifiably valid.

\section{Related work}\label{sec:related}

\subsection{Islamic geometric patterns and computational ornament}\label{sec:rel-islamic}

Islamic geometric patterns have been studied extensively from both historical and mathematical perspectives. Bonner \cite{bonner2017} describes the polygonal construction techniques used to generate stars and rosettes from underlying tessellations, while Broug \cite{broug2019} presents related methods from a practitioner's point of view. Abas and Salman \cite{abas1995} analyse these patterns through their symmetry groups, Necipo\u{g}lu \cite{necipoglu1996} connects the tradition to surviving architectural scrolls, and Gr\"unbaum and Shephard \cite{grunbaum1987} provide the broader mathematical theory of tilings and planar patterns. Lu and Steinhardt \cite{lu2007} further show the sophistication of decagonal and quasi-crystalline girih designs in medieval architecture. This literature motivates our treatment of Islamic ornament as a geometric object rather than as a visual texture.

Most computational work in this area has been procedural. Kaplan \cite{kaplan2000} introduced computer-generated Islamic star patterns, Kaplan and Salesin \cite{kaplan2004} developed the Najm system for designing star patterns in absolute geometry, and Kaplan \cite{kaplan2005} generated star patterns from polygons in contact. Ranjazmay Azari et al. \cite{ranjazmay2023} survey application-based principles and computational directions for Islamic geometric patterns. Grammar-based approaches also explicitly encode construction rules. For example, Sayed et al. \cite{sayed2016} use an auto-parameterised shape grammar to construct motif-based structures from a compact representation of their generating rules.

These approaches are largely correct because the construction rules are specified by hand. Our work is complementary. It learns completion from examples, but correctness is enforced through the orbital organisation of the candidate geometry and a symmetry-structured predictor, rather than through a fully hand-coded grammar or rule set. Procedural and grammar-based systems remain the appropriate reference for faithful reproduction of documented pattern families. A learned model, by contrast, can interpolate within a distribution of examples and can operate when the input geometry is incomplete.

\subsection{Neural generation, vector graphics, and knowledge-guided visual modelling}\label{sec:rel-neural-geometry}

Neural image and vector-graphics models can produce visually plausible results, but they rarely guarantee adherence to specific construction rules. Our setting is therefore knowledge-guided. Constructional geometry, orbit closure, fixed anchors, and symmetry-valid decoding restrict the output space before the network produces a completion.

The problem of reliable inference from incomplete or degraded input appears in several areas of visual modelling. Face-recognition systems, for example, must handle partial, rotated, or low-quality faces, and their reliability can degrade under such conditions \cite{elmahmudi2019deep}. Restoration-based methods have also been evaluated under forensic image degradations \cite{ugail2026forensic}. When reference data are scarce, one-class methods over compact features can test whether an input is consistent with a known source, as in the verification of historical sketches \cite{ugail2026sketches}. This is, in spirit, related to asking whether a generated completion is consistent with a visual tradition. Structured representations have also proved useful in multimedia tasks, such as coarse-to-fine alignment in person search \cite{huang2023beyond}, where explicit modelling of structure can be more effective than relying solely on a global image descriptor. Although our input and output are different, namely sparse construction geometry and a vector graph, the central concern is similar: neural inference must remain reliable when the available evidence is incomplete.

A further line of work generates vector content directly. SketchRNN represents drawings as stroke sequences \cite{ha2018}, DeepSVG models vector graphics hierarchically \cite{carlier2020}, differentiable rasterisation connects vector parameters to raster losses \cite{li2020}, and Im2Vec synthesises vector graphics without vector supervision \cite{reddy2021}. Our output is also vector-valued, but it is not a free sequence of strokes or paths. It is a structured graph selected from a candidate lattice, with symmetry and refinement constraints imposed directly on that graph. The bounded curve refinement used here, in which a learned per-edge parameter displaces an inserted control point within a fixed bound, is related to learned subdivision operators that apply constrained local deformations while preserving interpolatory structure, including neural tension operators for curve subdivision across constant-curvature geometries \cite{ugail2026neuraltension}. We adopt the same principle of bounded, interpretable refinement but tie it to rotational orbits so that it remains exactly symmetric.

The structural nature of the problem also connects this work to geometric deep learning and equivariant neural networks. Geometric deep learning provides a framework for computation on non-Euclidean and symmetry-structured domains \cite{bronstein2017,han2025}. Group-equivariant convolutional networks use group actions to share parameters across transformed inputs \cite{cohen2016}, with related theory for compact groups \cite{kondor2018}, planar rotation and reflection groups \cite{weiler2019}, and Euclidean-equivariant graph networks \cite{satorras2021}.

Our motivation overlaps with this literature, but the mechanism is different. We do not require network weights to commute with a group action on a grid, sphere, or point cloud. Instead, we organise candidate edges into rotational orbits and enforce consistency directly on the edge scores and refinements. The guarantee therefore applies to the discrete output graph, rather than only to an equivariant feature map inside the network. Even an equivariant backbone would still require orbit-level tying or projection at the output, because a threshold or top-$K$ selection rule could otherwise divide an orbit. For this reason, the relevant comparison in this work is not with a feature-space-equivariant baseline but with output-level orbit tying, output-level projection, and unconstrained decoding.

\subsection{Knowledge-constrained graph learning and trustworthy neural prediction}\label{sec:rel-graph-constraints}

Graph neural networks are widely used for learning over relational structures, including message-passing networks, attention-based models, and inductive representation-learning methods \cite{khemani2024gnnreview}. Our predictor builds on this general machinery, including graph convolution \cite{kipf2017}, inductive graph representation learning \cite{hamilton2017}, and graph attention \cite{velickovic2018}. However, the task is not open-ended graph generation in the sense of GraphVAE \cite{simonovsky2018}, GraphRNN \cite{you2018}, or discrete graph diffusion \cite{vignac2023}, where topology is generated within a broad combinatorial space. In our setting, the candidate edges are fixed in advance by the input control geometry and the admissible construction rules. The task is therefore structured completion: the network scores candidate edges and refinements, and the decoder selects a valid subset of the lattice.

The guarantee sought in this paper also places the work within constraint-aware learning. Much of that literature encourages predictions to respect domain knowledge through soft or optimisation-based mechanisms. Examples include semantic-based regularisation, which introduces logical constraints into the learning objective \cite{diligenti2017}, and primal-dual formulations that enforce constraints on output labels during training \cite{nandwani2019}. Knowledge-driven artificial intelligence more generally incorporates structured relational, symbolic, or domain knowledge into neural prediction. Knowledge-graph representation learning is one such route, where relations among entities shape both representation learning and downstream inference \cite{li2023kgshfp}.

The knowledge used in our work is geometric rather than semantic. Vertices, candidate edges, cyclic orbits, admissibility rules, and anchor constraints define the space in which completion is performed. The model is therefore knowledge-constrained without relying on a conventional textual knowledge graph. It is also stricter than many soft-constraint approaches: the orbit-tied predictor satisfies its symmetry, anchor-preservation, and bounded-refinement requirements by construction. We also evaluate inference-time projection as a natural post hoc repair, which allows us to separate the value of the orbit structure from the particular mechanism used to enforce it.

This work therefore lies at the intersection of knowledge-constrained learning, graph prediction, and structured visual modelling. It treats Islamic geometric ornament as a vector-geometric completion problem in which exact symmetry is not a stylistic preference, but a structural requirement encoded directly in the geometry.

\section{Problem formulation}\label{sec:problem}

\subsection{Control geometry, candidate lattice, and orbit structure}\label{sec:formulation}

A problem instance is a tuple $S=(P,N,z,G^{\star})$ in which $P$ is the sparse control geometry, $N$ is the target rotational symmetry order, $z$ is an optional motif-regime label, and $G^{\star}=(V,E^{\star})$ is the target completed pattern graph. The control geometry consists of a centre, a small set of anchor vertices, and boundary markers derived from traditional construction layouts, and it intentionally carries far less information than the completed pattern. We use the term 'control geometry' throughout and 'construction geometry' when emphasising the link to historical construction practice.

Formulating the task in this way makes the domain knowledge explicit. The sparse control geometry, the admissible candidate lattice, the cyclic orbit partition, and the anchor constraints together define the search space in which the completion is performed. Given $P$ and $N$, the lattice is generated deterministically. It comprises a vertex set $V$ of ring-structured points and a candidate edge set $E_c=\{e_1,\dots,e_M\}$ that enumerates every geometrically admissible connection: ring edges, radial edges, star-polygon chords, and secondary connectors. Each edge carries its endpoints, an edge-type label, geometric features, and an orbit identifier. By construction $E_c$ is closed under the cyclic rotation group $C_N$, so the group acts on candidate edges by $g\cdot e_i=e_j$ for $g\in C_N$, and the orbit of an edge is,
\begin{equation}
\mathcal{O}(e_i)=\{\,g\cdot e_i \,\colon\, g\in C_N\,\}.
\end{equation}
The orbits partition $E_c$. A completed pattern is exactly $N$-fold symmetric precisely when its selected edges form a union of complete orbits and its refinement parameters are constant on each orbit. This orbit structure is central to the paper, as both the guarantee and the experiments are phrased in terms of it.

\subsection{Completion objective and evaluation criteria}\label{sec:objective}

The model maps the featurised candidate lattice to an edge score $s_i$ and a bounded refinement parameter $\alpha_i$ for every candidate edge. A selection policy converts scores into a binary selection $\hat{y}_i$, either by thresholding or by retaining the top-$K$ scoring orbits at a calibrated ratio. The refinement parameter displaces the midpoint of a selected non-structural edge along its normal by $\tanh(\alpha_i)\,\alpha_{\max} L_i$, where $L_i$ is the edge length. Training minimises a weighted binary cross-entropy over candidate edges, together with a count-calibration term that matches the expected number of selected edges to the target and a mild regulariser on the refinement magnitudes. Selection policies are calibrated on validation data only, and no target information is used at inference.

The evaluation keeps fidelity and validity separate. Fidelity is the edge selection $F_1$, with its precision and recall. Calibration is the edge-density error between the predicted and target selections. Validity is measured chiefly by the count of selected-orbit violations, that is, orbits left only partly selected, together with the within-orbit spread of scores and refinements, any drift in the anchor vertices, and any refinement exceeding the bound $\alpha_{\max}$. For the rendered output, we additionally measure rotational self-Chamfer distance as a check on geometric symmetry. Robustness consists of re-evaluating all of these quantities with the input control geometry corrupted. We fix this protocol before any experiments because the paper's central argument is that fidelity alone cannot certify a completion in this domain.

\section{Method}\label{sec:method}

\subsection{Pipeline and candidate geometry construction}\label{sec:pipeline}

Candidate construction is the stage at which the geometry enters the learning problem. Rather than allowing the network to invent arbitrary shapes, we derive the admissible vertices, edges, edge types, and orbit identifiers from the control geometry and the chosen symmetry order. Figure~\ref{fig:schematic} sets out the pipeline. Sparse control geometry enters on the left. The candidate lattice is then constructed and organised into rotational orbits; geometric features are extracted for every vertex and edge; a graph network scores the edges and predicts refinements; a symmetry-structured decoder produces the completed graph; and a renderer converts it into strapwork-style vector ornament. At no stage does the model operate on pixels; it works on candidate geometry throughout.

\begin{figure}[t]
\centering
\includegraphics[width=\linewidth]{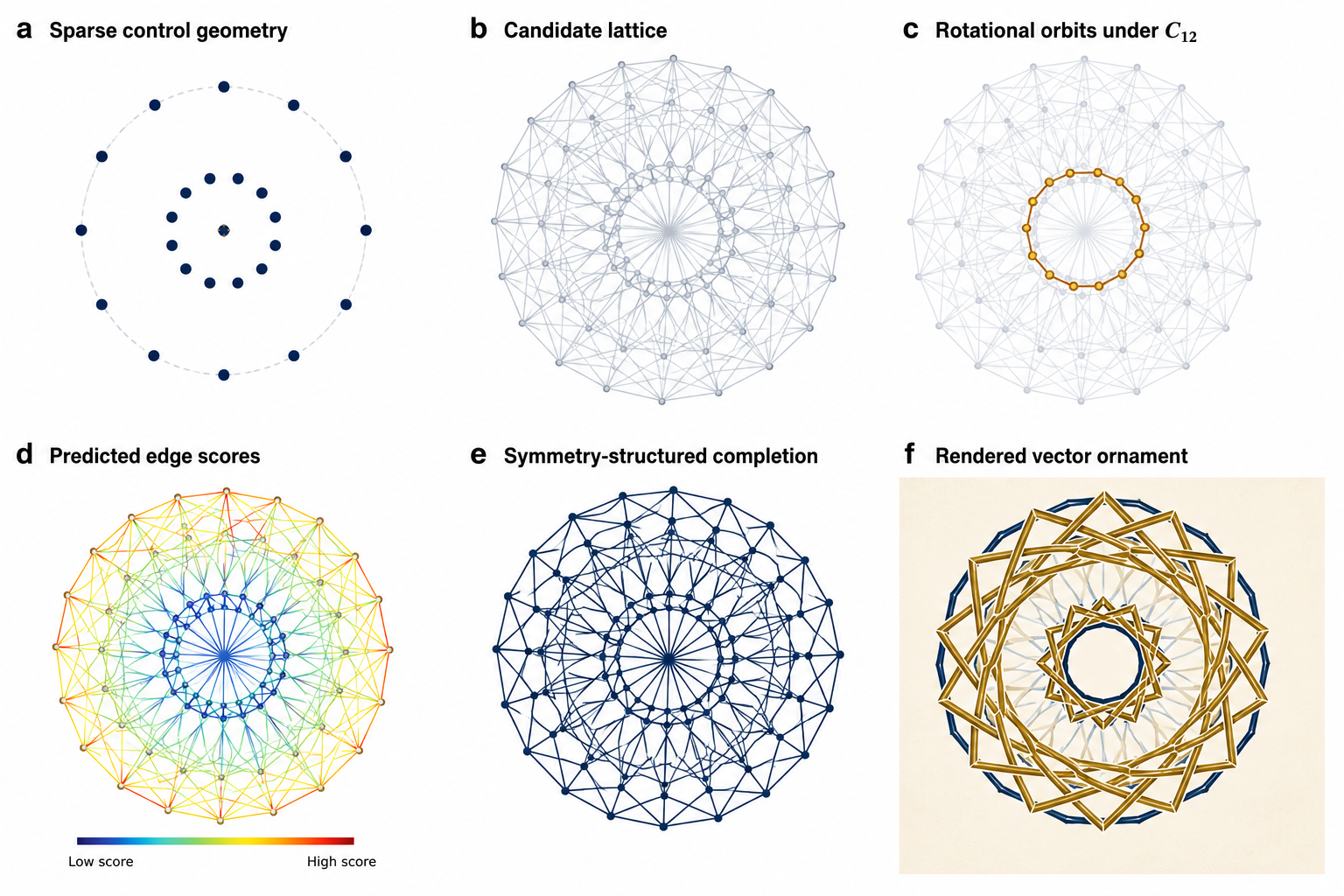}
\caption{Overview of symmetry-structured neural completion. Panel a shows the sparse control geometry, panel b the generated candidate lattice, panel c the organisation of candidate edges into rotational orbits under $C_N$, panel d the neural prediction of edge scores and bounded refinements, panel e the symmetry-structured completion, and panel f shows the rendered vector ornament.}
\label{fig:schematic}
\end{figure}

Constructing the candidate set proceeds in three steps. Vertices are placed on concentric rings whose radii and angular phases follow from the control points, so that each vertex acquires a ring index, a phase index, and a sector index. Candidate edges are then enumerated between admissible vertex pairs and labelled by type: ring edges, radial spokes, star-polygon chords that skip a fixed number of points on a ring, diagonal and secondary connectors, and boundary elements. Finally, each edge receives an orbit identifier derived from its type and the ring and phase labels of its endpoints, which renders it independent of the actual coordinates. Anchor vertices from the control geometry are marked and remain fixed throughout all subsequent stages. Because the vertices lie at exact multiples of $2\pi/N$, the edge set is closed under $C_N$ by construction, and this closure is the property on which the guarantee in Section~\ref{sec:theory} depends.

\subsection{Neural predictor and symmetry-structured decoding}\label{sec:decoding}

The predictor is a message-passing network over the candidate lattice. Vertex features encode position, radius, angle, anchor status, and the ring and phase labels; edge features encode displacement, length, radial and angular differences, the symmetry order, and a one-hot edge type. Several rounds of message passing update the vertex and edge representations, after which a score head produces $s_i$ and a refinement head produces $\alpha_i$ for every edge, optionally conditioned on the symmetry order and the motif-regime label. We make no claim to novelty for the backbone. What matters is how its outputs are bound to the orbit structure. The network predicts within a knowledge-defined space, and the decoding regime determines whether that knowledge is enforced intrinsically, by projection, or not at all.

Four decoding regimes are built on this single predictor. The orbit-tied regime averages the score logits and refinement outputs over each orbit during the forward pass, enforcing $s_i=s_j$ and $\alpha_i=\alpha_j$ whenever $e_i$ and $e_j$ lie in the same orbit, and it does so during both training and inference. The selection-projection regime trains the network without any tying and aggregates scores to the orbit level only at the moment of selection, which tests whether exact topological symmetry can be recovered as a purely inference-time operation. Full projection goes further, averaging both the scores and the refinements over orbits at inference; it is the strongest post-hoc baseline, since it replicates the entire output structure of the orbit-tied model. The unstructured free regime serves as the negative baseline: the network selects edges independently with its own calibrated policy and uses no orbit information. All four share a common backbone, data, and training procedure, so any differences among them are attributable to the symmetry structure alone.

\subsection{Bounded refinement and ornament rendering}\label{sec:render}

We keep the refinement deliberately conservative. Structural edge types remain straight. Ornamental edges are refined by displacing the midpoint along the edge normal by $\tanh(\alpha_i)\,\alpha_{\max} L_i$ and interpolating a quadratic curve through the endpoints. Endpoints are never moved, anchors remain fixed, and the $\tanh$ saturation enforces the bound $\lvert\delta_i\rvert\le\alpha_{\max} L_i$ regardless of the network output, which underwrites the refinement clause of the guarantee.

The renderer is purely cosmetic. It converts the predicted graph into a strapwork visualisation in which bands receive a background-coloured casing so that crossings read as woven, and a fixed visual hierarchy maps the predicted edge types to gold star-polygons, indigo frame straps, and thin pale connectors. Every selected edge is drawn, and none is added, removed, or moved, so the woven appearance arises solely from stroke order. The same geometry is exported as scalable vector graphics, so rendering affects only the output's readability, not its topology or any guarantees.

\section{Construction guarantee}\label{sec:theory}

This section makes explicit how the embedded geometry constrains the output space. The guarantee does not require the learned weights to be accurate, since it follows from the orbit structure and the validity-preserving decoder rather than from the network itself. It rests on four standing assumptions, all of which hold by construction in our system: $E_c$ is closed under $C_N$; the orbit identifiers are correct; the orbit-tied model averages scores and refinements over orbits during the forward pass; and selection operates on orbit-level scores. Under these assumptions, the guarantee below holds. It states the invariants enforced by the orbit-tied decoder and is a consequence of the architecture rather than a statement about learning.

\begin{proposition}\label{thm:validity}
For any input control geometry, any symmetry order $N$, and any selection policy applied to orbit-level scores, the orbit-tied completion satisfies four properties. The selected edge set is a union of complete rotational orbits and therefore exhibits zero selected-orbit violations. The refined geometry rendered on the canonical lattice is exactly $N$-fold rotationally symmetric. Every anchor vertex of the control geometry is preserved exactly. Every refinement displacement satisfies $\lvert\delta_i\rvert\le\alpha_{\max} L_i$.
\end{proposition}

\noindent This constitutes trustworthiness of a knowledge-constrained kind: the network may still return a poor prediction, but whatever it returns obeys the formal geometric rules encoded in the lattice and the orbit structure.

\begin{proof}
Orbit averaging makes the score of every edge in an orbit identical, so any selection rule that depends only on these orbit-level scores, which are the scalars consumed by the selection policy, either selects the whole orbit or none of it, which yields the first property. The scores consumed by the selection policy are pre-sigmoid logits, and orbit averaging is applied to these logits, so the equality of within-orbit scores is exact and is unaffected by any monotone link function applied afterwards. For the second, let $g\in C_N$ and let $e_j=g\cdot e_i$. Closure of the lattice under $C_N$ maps the endpoints of $e_i$ to those of $e_j$. The selection of $e_i$ implies the selection of $e_j$ by the first property, and the tied refinement gives $\alpha_j=\alpha_i$, so the refined curve of $e_j$ is exactly the rotation by $g$ of the refined curve of $e_i$. The union of refined curves is therefore invariant under $C_N$. The third property holds because no stage of the pipeline modifies vertex positions and anchors are excluded from refinement. The fourth follows from the $\tanh$ saturation of the refinement head.
\end{proof}

The guarantee is unconditional: it requires no clean input, no trained network, no particular selection policy, and no post hoc step. It concerns validity rather than accuracy. The predictor may well select a wrong but exactly symmetric completion, which is precisely why fidelity, calibration, and validity are measured separately in Section~\ref{sec:results}.

An unconstrained network admits a closely related repair. Projecting its scores to orbit level before selection enforces the first property, and projecting the refinements as well enforces the second. The constructive model possesses the guarantee intrinsically, whereas projection is a repair that must be remembered and applied. We therefore compare the two on equal terms rather than treating projection as an invalid baseline, and Section~\ref{sec:res-equiv} reports the outcome of that comparison.

\section{Experimental setup}\label{sec:setup}

All experiments are conducted on synthetic, Islamic-geometry-inspired pattern graphs generated by a stochastic procedure over the candidate lattice, with motif families and neutral procedural-regime labels controlling the target distribution. The main configuration uses six hundred training graphs and separate validation and test splits of one hundred and twenty each, across symmetry orders six, eight, ten, and twelve and five motif families. One held-out experiment trains on orders six, eight, and twelve only and evaluates at order ten. Every experiment is repeated over three random seeds, and results are reported as means.

Six methods are compared, and Table~\ref{tab:variants} summarises the role of each. The orbit-tied model and the three decoding regimes for the unconstrained network are as in Section~\ref{sec:decoding}. The procedural template scores candidate edges using hand-designed rules derived from the control geometry and is symmetric by construction, making it a strong non-learned reference. The nearest-neighbour baseline copies the target of the closest training instance. The two neural models are trained identically, with a warm-up and cosine learning-rate schedule, gradient clipping, and restoration of the best validation checkpoint, for 70 epochs at a peak learning rate of $7\times10^{-4}$, on a message-passing backbone of hidden width 96 and 3 rounds. Selection policies are calibrated on validation data, which places the structured policies at orbit-level top-$K$ ratios near $0.4$ and the unstructured policy at free ratios near $0.34$.

\begin{table}[t]
\centering
\caption{Method variants and the question each one answers. Guarantee refers to the exact selected-orbit consistency of the output.}
\label{tab:variants}
\begin{tabular}{@{}lllll@{}}
\toprule
Method & Training constraint & Inference projection & Guarantee & Purpose \\
\midrule
Orbit-tied & orbit averaging & none needed & yes & constructive guarantee \\
Projection (selection) & none & scores to orbits & yes & post-hoc topology repair \\
Projection (full) & none & scores and $\alpha$ & yes & strongest post-hoc baseline \\
Unstructured (free) & none & none & no & negative baseline \\
Procedural template & rule-based & symmetric by design & yes & non-learned reference \\
Nearest neighbour & retrieval & inherits target & yes & memorisation check \\
\bottomrule
\end{tabular}
\end{table}

Robustness is tested under three corruption modes that affect only the network's input. The ground truth, the canonical lattice, and its exact geometric orbit partition all remain clean, so a violation is always counted against the same orbit structure that defines a valid completion. Coordinate jitter perturbs every point independently with Gaussian noise of relative standard deviation up to $0.15$. Missing control geometry zeroes the features of a random fraction of vertices and their incident edges, up to $30\%$. Single-sector corruption jitters the points of one angular sector only, the most asymmetric form of localised corruption. In each case the question is whether symmetry-structured inference remains accurate and valid when the input geometry is imperfect, which is the situation any practical completion system faces.

For the training-control study in Section~\ref{sec:res-control}, we additionally train augmented versions of both predictors. Every training input is then corrupted by the dropout mechanism with a severity uniformly drawn from [0, 30\%], while the targets, the lattice, and the orbit partition are left untouched, and all other training settings remain unchanged.

\section{Results}\label{sec:results}

\subsection{Fidelity on clean inputs}\label{sec:res-fidelity}

Both neural models train stably on every seed, with training and validation losses moving together, so none of the comparisons below is an optimisation artefact. Figure~\ref{fig:mainf1} and Table~\ref{tab:main} report edge-selection fidelity on the in-distribution test set. The orbit-tied model attains a mean $F_1$ of $0.519$, a selection projection of $0.523$, and an unstructured free decoding of $0.520$, with seed-level standard deviations of around $0.015$. The procedural template reaches $0.534$ and the nearest-neighbour baseline stalls at $0.437$, so the task is not solved by memorisation. As every method is evaluated on the same test patterns, the neural variants can be compared through a paired per-pattern analysis rather than by inspecting overlapping error bars.

Across the three hundred and sixty paired patterns, the orbit-tied minus unstructured difference averages $-0.001$, with a bootstrap $95\%$ interval of $[-0.009, 0.007]$ and a Wilcoxon signed-rank $p$ of $0.84$. The orbit-tied minus projection difference averages $-0.004$, with an interval of $[-0.012, 0.004]$ and $p=0.77$. The intervals are obtained from a nonparametric bootstrap of the paired differences over 10,000 resamples, and the Wilcoxon tests use the same differences. The equivalence margin was set at two $F_1$ points in advance, approximately the seed-level standard deviation of the predictors, and both intervals fall entirely within it. Enforcing exact validity therefore incurs no measurable fidelity penalty relative to an unconstrained predictor, and we claim no overlap advantage for the structured methods.

\begin{figure}[t]
\centering
\includegraphics[width=0.9\linewidth]{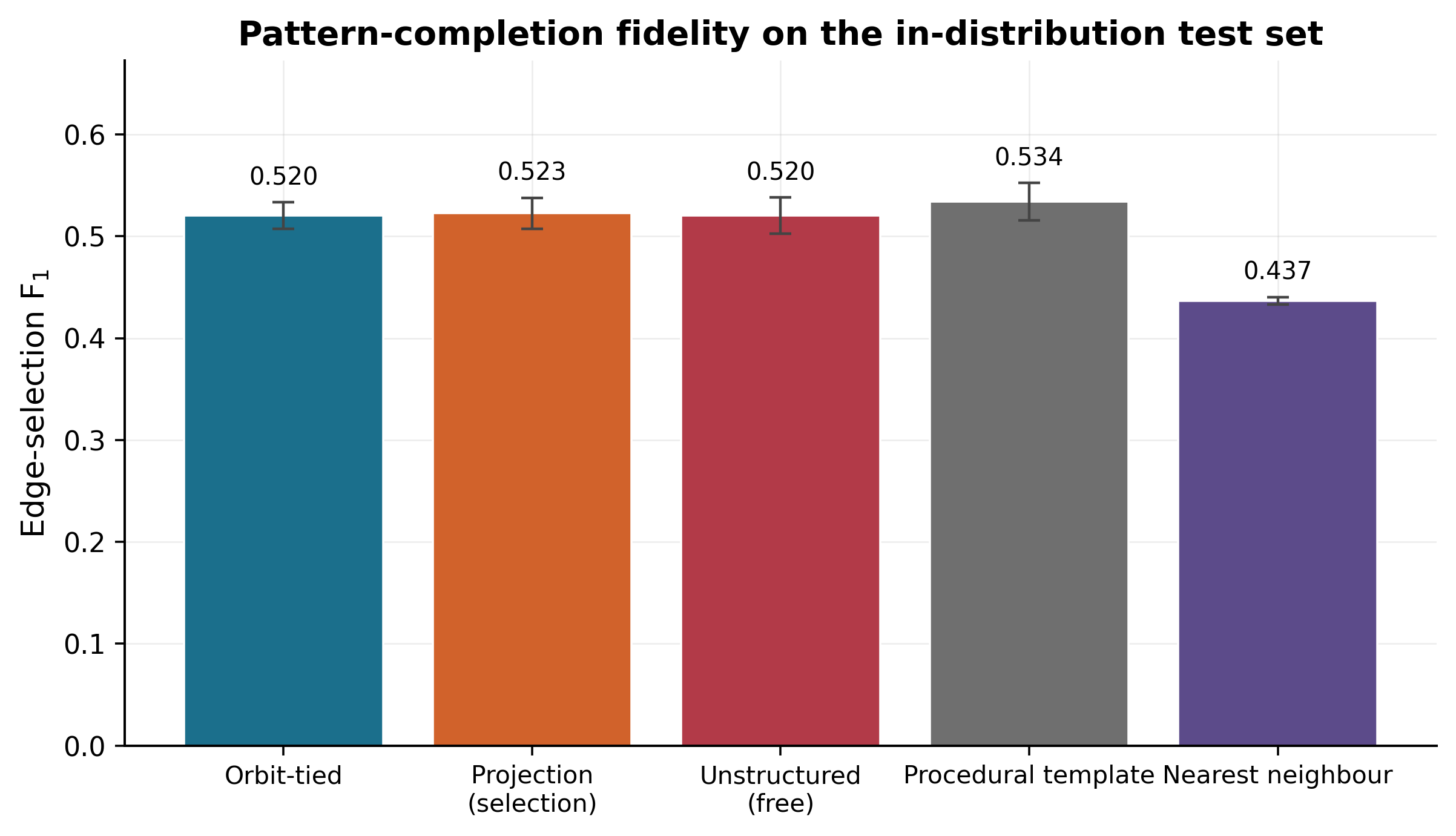}
\caption{Edge-selection fidelity on the in-distribution test set, shown as means over three seeds with standard deviations. All neural variants perform comparably. The template attains the highest raw overlap, and Section~\ref{sec:res-template} shows that it does so only through severe over-completion.}
\label{fig:mainf1}
\end{figure}

\begin{table}[t]
\centering
\caption{Main quantitative results on the clean test set, reported as means with standard deviations over three seeds. Violations count selected-orbit violations per pattern, and density error is the relative deviation of the predicted edge count from the target count, so it reflects the predicted-to-target ratio. The template attains its leading $F_1$ score through low precision and high recall, with a roughly threefold density error, as analysed in Section~\ref{sec:res-template}.}
\label{tab:main}
\begin{tabular}{@{}lccccc@{}}
\toprule
Method & $F_1$ & Precision & Recall & Density error & Violations \\
\midrule
Orbit-tied & $0.519 \pm 0.015$ & $0.498 \pm 0.025$ & $0.575 \pm 0.002$ & $0.088 \pm 0.005$ & $0$ \\
Projection (selection) & $0.523 \pm 0.015$ & $0.496 \pm 0.029$ & $0.588 \pm 0.009$ & $0.092 \pm 0.008$ & $0$ \\
Unstructured (free) & $0.520 \pm 0.018$ & $0.495 \pm 0.024$ & $0.581 \pm 0.020$ & $0.090 \pm 0.006$ & $3.3$ \\
Procedural template & $0.534 \pm 0.019$ & $0.422 \pm 0.024$ & $0.778 \pm 0.005$ & $0.268 \pm 0.018$ & $0$ \\
Nearest neighbour & $0.437 \pm 0.004$ & $0.454 \pm 0.015$ & $0.445 \pm 0.012$ & $0.086 \pm 0.005$ & $0$ \\
\bottomrule
\end{tabular}
\end{table}

\subsection{Structural validity}\label{sec:res-validity}

Figure~\ref{fig:validity} presents the validity audit. Across all seeds and all three hundred and sixty test patterns, the orbit-tied model records zero selected-orbit violations, zero anchor drift, and zero excess beyond the refinement bound, exactly as Proposition~\ref{thm:validity} requires. The audit is label-free: it rotates each selected edge by exactly $2\pi/N$ and checks that the corresponding image edge is also selected, so it does not rely on the system's internal orbit identifiers. Selection projection attains the same zero count by construction.

The unstructured model behaves quite differently. It violates orbit consistency on $3.3$ orbits per pattern on average, even on clean input, while matching the structured methods on $F_1$. This is the first half of the evaluation argument. Two methods can be indistinguishable on overlap yet differ categorically on validity, and a metric suite without a validity term cannot detect the difference. Statistical fidelity and knowledge validity are thus not the same thing: the unstructured model can match the structured methods on overlap-based $F_1$ while still violating the geometry that defines a valid pattern.

Figure~\ref{fig:pair} illustrates this on a single example. The structured and free completions of the same control geometry cover approximately the same edges, and only the structured one is exactly symmetric. On rotational self-Chamfer distance, the rendered outputs agree to within the resolution of a point-sampled metric, placing the orbit-tied model at the sampling floor of $0.006$, free decoding at $0.009$, and the template at $0.010$.

\begin{figure}[t]
\centering
\includegraphics[width=0.85\linewidth]{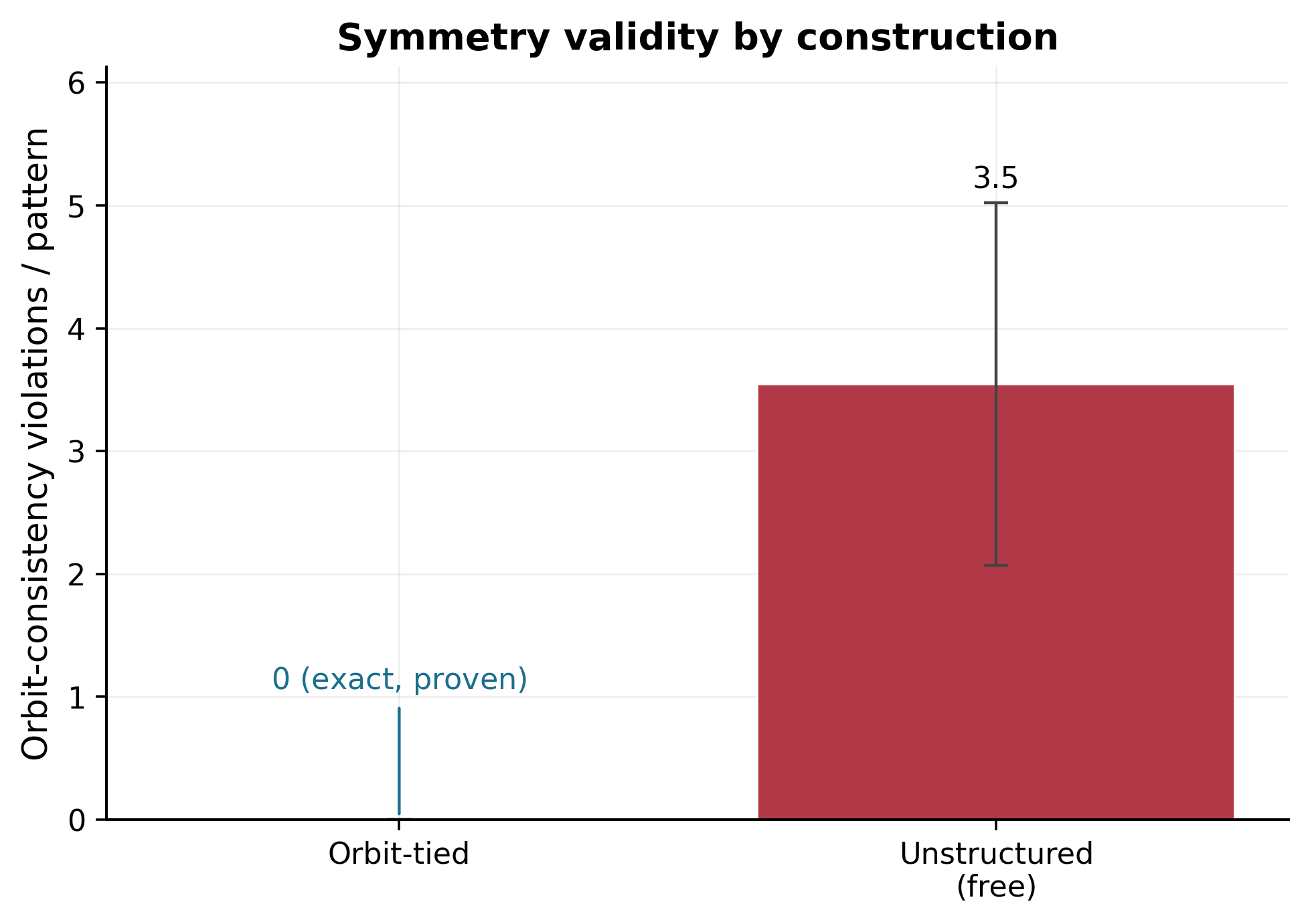}
\caption{Structural validity audit. Orbit-tied decoding eliminates selected-orbit violations entirely, whereas unstructured decoding can violate orbit consistency even on clean input. Selection projection achieves the exact same zero count and is omitted from the chart solely to maintain legible contrast.}
\label{fig:validity}
\end{figure}

\begin{figure}[t]
\centering
\includegraphics[width=0.95\linewidth]{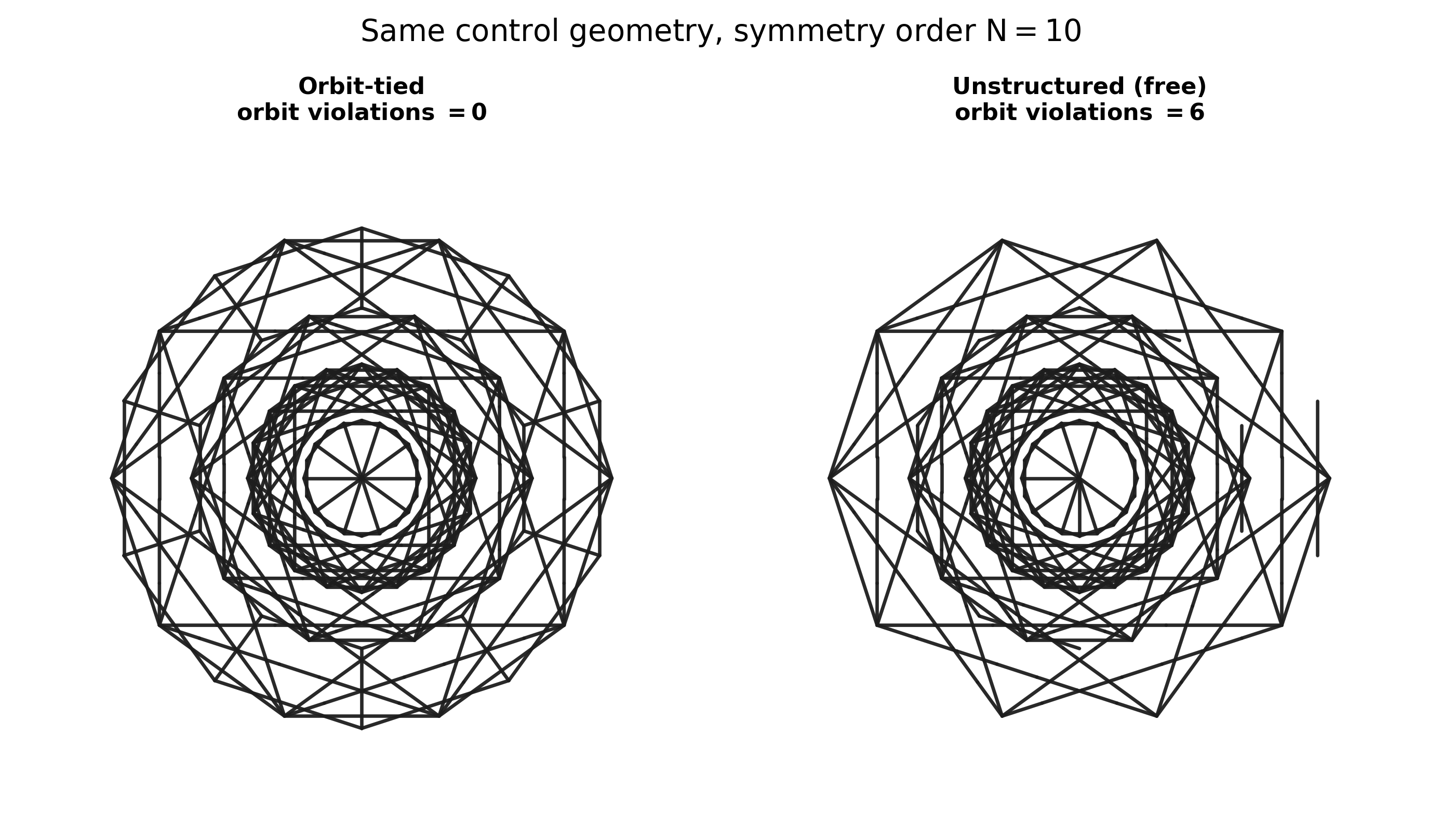}
\caption{Completions of the same control geometry at the held-out order $N=10$. Similar edge-level fidelity can mask structural invalidity, since the unstructured output violates orbit consistency, whereas the orbit-tied output is exactly symmetric.}
\label{fig:pair}
\end{figure}

\subsection{Robustness and the corrupted-input training control}\label{sec:res-robust}\label{sec:res-control}

Figure~\ref{fig:robust} presents the corruption study, and Table~\ref{tab:robust} summarises the strongest condition. Coordinate jitter leaves the fidelity of every method essentially unchanged, even at the largest perturbation, although the free model's violations rise from $3.3$ to about $7$ per pattern, so coordinate noise is not the principal failure mode. Single-sector corruption is similarly benign because orbit-level aggregation distributes a local disturbance across the whole orbit for the structured methods, and the free model is only mildly affected.

Missing control geometry is where the methods diverge. With $30\%$ of the input removed, the free model drops from $0.524$ to $0.415$ in $F_1$, losing precision and recall, and its violations increase from $3.3$ to $37.5$ per pattern. The structured methods move only slightly, from $0.518$ to $0.502$ for orbit-tied and from $0.527$ to $0.509$ for full projection, and both retain zero violations at every severity. The template scarcely accounts for the corrupted features and remains at $0.535$, with its characteristic over-completion. For clean-trained models, then, symmetry-structured inference is far less sensitive to missing geometry and remains exactly valid, although, as it emerges, the fidelity and validity components of this contrast have different causes.

Missing geometry probes whether the predictor remains reliable when part of its knowledge-bearing input is absent. The natural question is whether the fidelity gap reflects the symmetry structure itself or merely a mismatch between clean training and corrupted testing, and whether the augmented variants resolve it. Figure~\ref{fig:control} reports the result, with the corresponding rows folded into Table~\ref{tab:robust}. Under missing-geometry augmentation, most of the unstructured model's fidelity is recovered, from $0.415$ to $0.517$ at $30\%$ missing, so the bulk of the collapse reflects the training distribution rather than the missing structure. The augmented unstructured model, in fact, loses less fidelity at the strongest corruption than the clean-trained orbit-tied model, so symmetry structure has no unique claim to fidelity robustness once training and test conditions are matched.

Validity is where the two diverge. The augmented unstructured model still violates orbit consistency on about 9 orbits per pattern at the strongest corruption, and on about 3 even on clean input. Combining both remedies, the augmented symmetry-structured variants form the strongest configuration overall, maintaining fidelity near $0.52$ across all severities with zero violations. The small clean-to-corrupted $F_1$ differences for these augmented models show no consistent direction across seeds and lie well within the seed-level variation, so they should be read as no measurable change rather than as improvement under corruption. The reason the clean-trained structured methods are insensitive to corruption is readily stated: orbit aggregation averages $N$ predictions per orbit, and a local disturbance is attenuated by the same redundancy the symmetry provides. The honest summary is that corrupted-input training and symmetry structure address different failure modes. Augmentation restores fidelity; symmetry structure guarantees validity; and the two compose.

\begin{figure}[t]
\centering
\includegraphics[width=\linewidth]{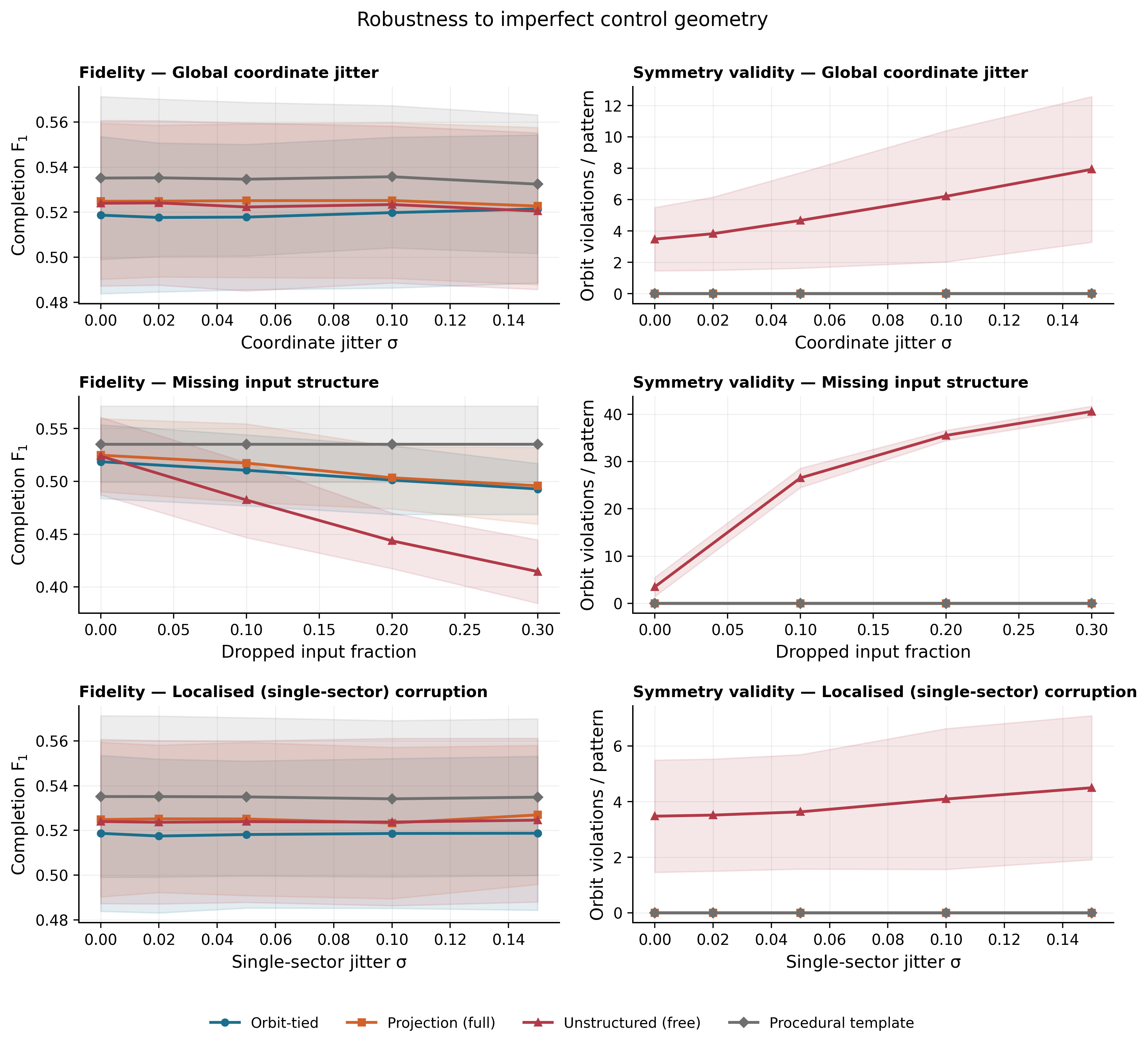}
\caption{Robustness to imperfect control geometry across three corruption modes. The left column reports fidelity, and the right column reports exact orbit violations. Missing input structure exposes the value of symmetry-structured inference, since the clean-trained free model loses both fidelity and validity while both structured methods remain close to their clean operating point with exactly zero violations. The training-control study reported in the same section separates the fidelity and validity components of this contrast.}
\label{fig:robust}
\end{figure}

\begin{table}[t]
\centering
\caption{Robustness under $30\%$ missing control geometry, reported as means over three seeds. Augmented variants are trained with missing-geometry augmentation and evaluated under the identical protocol.}
\label{tab:robust}
\begin{tabular}{@{}lcccc@{}}
\toprule
Method & Clean $F_1$ & Corrupted $F_1$ & Drop & Violations \\
\midrule
Orbit-tied & $0.518$ & $0.502$ & $0.017$ & $0$ \\
Projection (full) & $0.527$ & $0.509$ & $0.018$ & $0$ \\
Unstructured (free) & $0.524$ & $0.415$ & $0.109$ & $37.5$ \\
Orbit-tied (augmented) & $0.520$ & $0.524$ & $-0.004$ & $0$ \\
Projection (full, augmented) & $0.521$ & $0.522$ & $-0.001$ & $0$ \\
Unstructured (augmented) & $0.520$ & $0.517$ & $0.003$ & $9.0$ \\
Procedural template & $0.535$ & $0.535$ & $0.000$ & $0$ \\
\bottomrule
\end{tabular}
\end{table}

\begin{figure}[t]
\centering
\includegraphics[width=\linewidth]{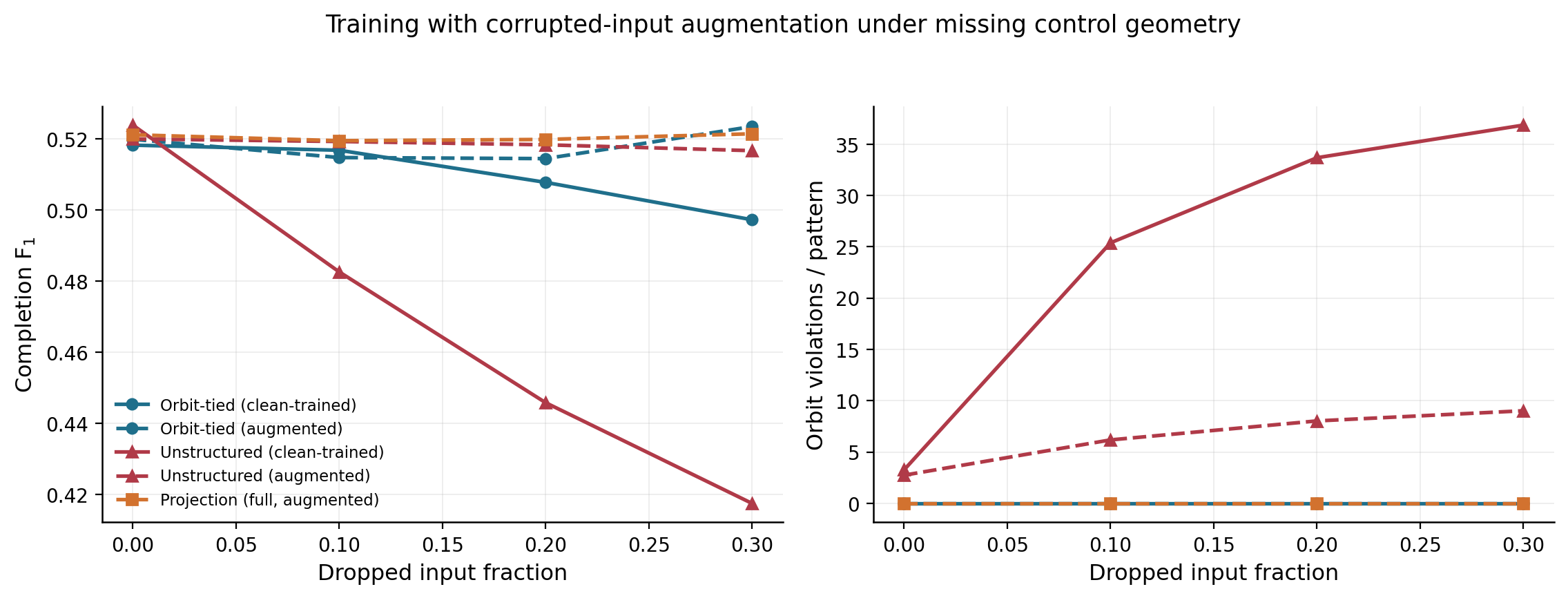}
\caption{Corrupted-input training control under missing control geometry. Augmentation largely restores the fidelity of the unstructured model but does not eliminate its orbit violations, while the augmented symmetry-structured variants hold fidelity essentially flat with exactly zero violations at every severity.}
\label{fig:control}
\end{figure}

\begin{figure}[t]
\centering
\includegraphics[width=\linewidth]{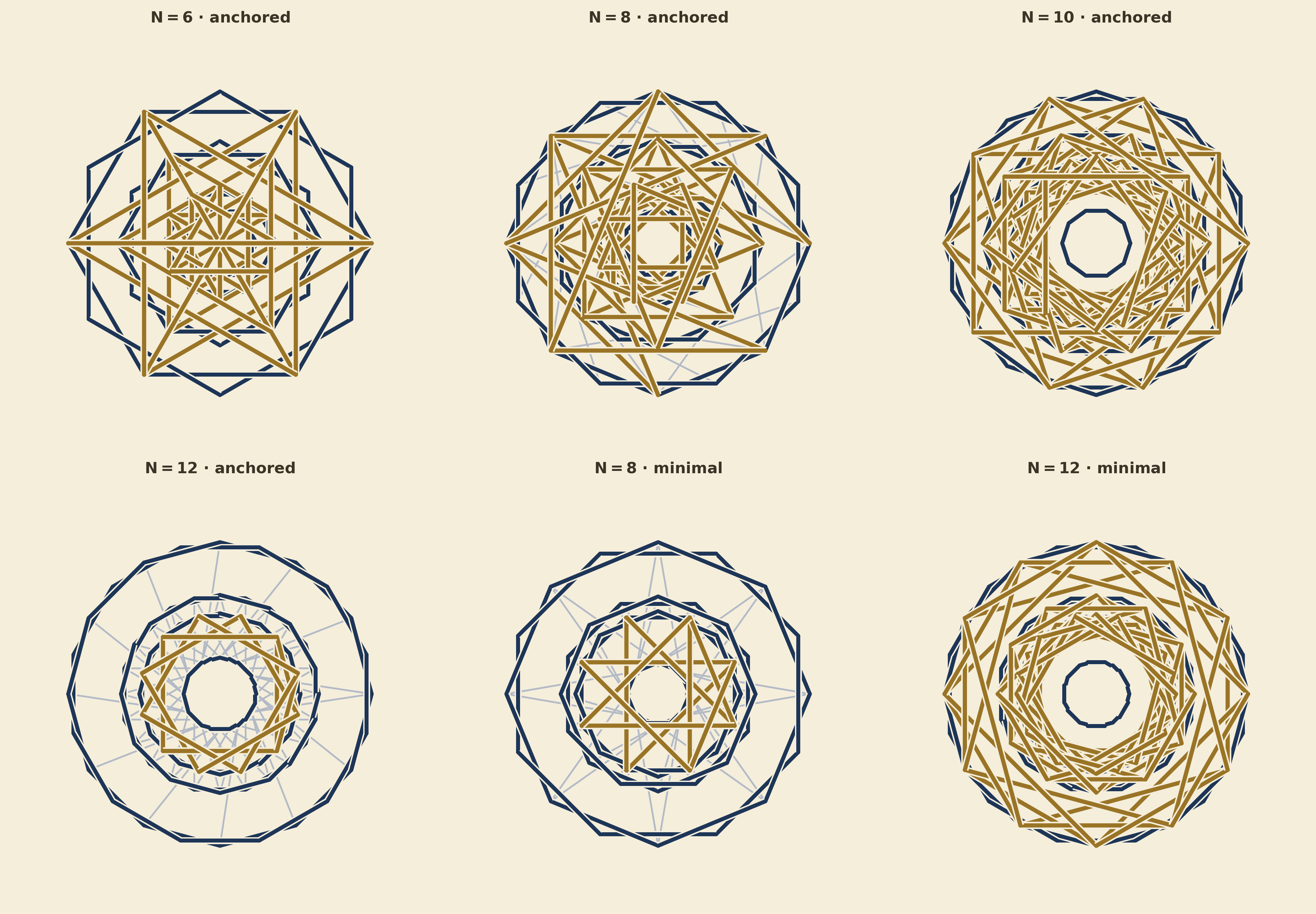}
\caption{Completed patterns rendered as scalable vector ornament without altering the predicted topology. The panels show curated completions across symmetry orders and motif families. The declared rendering hierarchy maps predicted star-polygon chords to gold bands, structural edges to indigo straps, and secondary connectors to thin pale lines.}
\label{fig:ornament}
\end{figure}

\subsection{Orbit tying versus projection, generalisation, and the template critique}\label{sec:res-equiv}\label{sec:res-template}

It is worth stating directly how orbit tying compares with inference-time projection, since the comparison shows that the operative factor is not any particular architectural device but the availability of an explicit orbit-level structure that can be enforced during decoding. Across all $42$ combinations of seed, corruption mode, and severity, the mean $F_1$ difference between the orbit-tied model and full projection is $-0.008$, with a standard deviation of $0.008$. It is positive in only five of the forty-two cells and remains slightly negative at the highest levels of corruption. Orbit-tied training therefore offers no fidelity advantage over projection in this domain, and we report the equivalence rather than obscure it. The result is informative in itself: identical validity and indistinguishable fidelity from two different enforcement mechanisms identify the candidate geometry and its orbit organisation as the operative ingredient, with the particular device secondary. What orbit tying retains, and projection does not, is an unconditional guarantee that holds for any orbit-level selection policy without an additional step, which matters wherever every output must be valid and a forgotten repair stage constitutes a failure mode.

Fidelity is also stable across symmetry orders. In the held-out experiment, with models trained on orders six, eight, and twelve and evaluated at order ten and the order supplied at inference, both the orbit-tied model and projection reach $0.524$. We describe this as conditional generalisation rather than zero-shot transfer, since the order is provided explicitly. Across training budgets from 100 to 600 graphs, the difference between structured and unstructured variants shows no consistent direction and never exceeds 1 $F_1$ point, so the constraint neither helps nor hinders learning at any budget tested.

The template baseline completes the evaluation argument begun in Section~\ref{sec:res-validity}. Its leading $F_1$ of $0.534$ rests on precision near $0.42$ and recall near $0.78$, with an edge-density error of about $0.27$, roughly three times the value of $0.09$ shared by the neural methods. The template simply selects far too many edges, and the overlap metric rewards the excess. Although $F_1$ penalises low precision through its harmonic form, it does not expose the structural or density profile of a completion, so a leading overlap score can conceal a qualitatively different output. Weighting precision more heavily makes this explicit: the $F_{0.5}$ score computed from the mean precision and recall in Table~\ref{tab:main} is about $0.46$ for the template against about $0.51$ for each of the three neural predictors, and the template leads only when recall is weighted equally with precision. Together with the validity audit, this establishes that assessing Islamic geometric completion requires calibration and validity measures in addition to overlap.

\subsection{Qualitative results}\label{sec:res-qual}

Figure~\ref{fig:ornament} concludes the evaluation with completed patterns rendered as vector ornaments across a range of symmetry orders and motif families. Every panel is exactly $N$-fold symmetric by Proposition~\ref{thm:validity}, and the strapwork rendering alters only the appearance, not the underlying geometry. The release pairs each styled panel with its plain topological completion so that the two may be compared directly, and includes scalable vector exports of every ornament figure for fabrication.

\section{Discussion}\label{sec:discussion}

\subsection{What symmetry structure provides and what it does not}\label{sec:disc-structure}

The results are best understood by separating the role of learning from the role of geometric knowledge. The neural network performs the statistical part of the task: it predicts which candidate edges and refinements are likely to complete the pattern. The geometric structure defines the admissible output space and specifies what counts as a valid completion. These two components are complementary. Learning improves fidelity, while knowledge-constrained decoding guarantees validity.

The experiments show that symmetry structure provides three main benefits. It gives exact validity supported by a construction proof; it gives deterministic behaviour under any orbit-level selection policy; and it preserves validity even when the input control geometry is corrupted. These benefits are obtained without any meaningful loss of fidelity within the pre-specified margin. At the same time, the structure does not give a fidelity advantage over inference-time projection, nor does it provide a consistent data-efficiency benefit or superior generalisation across symmetry orders. The results therefore locate the contribution of symmetry structure quite precisely: its value lies in exactness. The training-control experiment reinforces this point. Augmenting the unstructured model with corrupted inputs recovers much of its lost fidelity under missing geometry, but it does not restore exact validity. The strongest configuration is obtained when corrupted-input training and symmetry-structured decoding are used together.

This equivalence between the constructive and projected versions is informative. It shows that, in this setting, the candidate geometry and its organisation into rotational orbits are more important than the particular mechanism used to enforce symmetry. Orbit tying remains the natural choice when every output must be certified, because its guarantee is intrinsic to the model and cannot be lost by omitting a post-processing step. Projection, however, is an effective alternative when one wishes to retrofit an existing unconstrained predictor without rebuilding the architecture.

The practical implications for Islamic geometric design are direct. Exact symmetry is not merely an aesthetic preference in this domain; it is a defining structural requirement. Vector output is also the natural format for construction, teaching, and fabrication, while sparse control geometry provides an interface that mirrors traditional construction practice. The robustness result is especially relevant here. Control geometry obtained from sketches, scans, or partial records will often be incomplete, yet the structured methods continue to produce valid completions under precisely this kind of deficiency.

\subsection{Limitations and future work}\label{sec:disc-limits}

Several limitations should be noted. All quantitative results are obtained from procedurally generated synthetic patterns, so the evaluated style range is limited by the generator, and no historical corpus is used in the quantitative experiments. The candidate lattice is also procedural, which restricts the model to the structures it can represent; some of these structures are visually dense. In addition, interlacing is currently rendered as a visual effect rather than predicted as part of the pattern topology, and the framework does not yet cover reflective or wallpaper-group symmetries. Inference-time projection remains competitive with the constructive model throughout the experiments, and no human aesthetic study has been conducted. We also do not claim superiority over procedural construction systems, which remain the appropriate reference for faithful reproduction of documented pattern families.

The robustness tests have further limits. The missing-geometry experiment removes vertices independently, whereas real damage may be spatially correlated, for example in the form of a continuous missing sector or a damaged scan region. The single-sector experiment addresses a related case, but broader forms of structured loss remain to be tested. The knowledge encoded in the present system is geometric and constructional rather than historical or iconographic. The framework therefore guarantees structural validity with respect to the specified lattice and orbit rules, but it does not certify historical authenticity. Such authentication would require a different form of evidence, more closely matching a completion to known sources through compact feature representations, as in the verification of historical sketches \cite{ugail2026sketches}. The outputs should therefore be understood as structurally valid completions inspired by Islamic geometric design, unless they are separately validated by a domain expert.

These limitations point to several directions for future work. Modelling strapwork and interlacing as topology would make the weave part of the prediction rather than a rendering effect. Extending the orbit machinery to dihedral and wallpaper groups would bring reflective and periodic ornament within scope. Further progress would also require learning from vectorised historical corpora, conducting preference studies of the generated ornaments, developing interactive tools based on editable control geometry, and validating the vector exports for fabrication. Together, these extensions would move the framework closer to practical design and heritage applications.

\section{Conclusion}\label{sec:conclusion}

In this paper, we have introduced a symmetry-structured neural framework for completing Islamic geometric patterns from sparse control geometry. The task is formulated not as image synthesis, but as vector-geometric completion: the model predicts edges and bounded refinements over a candidate lattice whose edges are grouped into rotational orbits under $C_N$. Exact symmetry is therefore enforced as explicit geometric knowledge rather than learned only from data.

The main finding is that exact structural validity can be imposed without a measurable loss in fidelity. The orbit-tied variant guarantees exact $N$-fold symmetry, fixed anchor points, and bounded refinements for any input and any orbit-level selection rule, and these properties are verified numerically. Inference-time projection achieves the same validity as a post hoc alternative, showing that the essential ingredient is the orbit structure of the candidate geometry rather than the particular enforcement mechanism.

The experiments clarify the role of symmetry structure. On clean inputs, structured and unstructured predictors achieve comparable $F_1$ scores, but only the structured methods eliminate orbit-consistency violations. When control geometry is missing, an unstructured decoder loses fidelity and accumulates symmetry violations; training with corrupted inputs recovers much of the lost fidelity but does not restore exact validity. Augmentation and symmetry structure, therefore, address distinct failure modes: augmentation improves robustness in fidelity, while symmetry structure guarantees validity. The procedural template further shows that overlap alone is insufficient, since a high score can result from selecting too many edges.

Overall, the results support symmetry-structured neural completion as a reliable route to scalable vector ornament in the Islamic geometric tradition. The framework produces outputs that are not merely plausible, but certifiably valid with respect to explicit geometric rules. The present work is limited to rotational symmetry and procedurally generated patterns. Future work will extend the approach to richer historical families, predictive modelling of strapwork and interlacing, and broader symmetry groups such as dihedral and wallpaper symmetries.

\section*{Acknowledgments}
The authors would like to thank the Centre for Visual Computing and Intelligent Systems at the University of Bradford for providing the computing resources used in this work.

\section*{Data availability}
The complete experimental code, the trained model checkpoints, all per-seed and per-condition numerical results, and the figure-generation pipeline are openly available at \url{https://github.com/ugail/Neural-Completion-of-Islamic-Geometric-Patterns}, together with the random seeds and cached artefacts required to reproduce every experiment reported in this paper without retraining.

\section*{Declarations}
The authors declare no competing interests. No specific funding was received for this work. H.U. led the conceptualisation of the research and the development of the methodology. I.M. conducted the experimental design and computational experiments. Both authors reviewed the results, discussed the experimental findings, contributed to the interpretation, and approved the final manuscript.


\begin{thebibliography}{99}

\bibitem{bonner2017}
Bonner, J. (2017). \textit{Islamic Geometric Patterns: Their Historical Development and Traditional Methods of Construction}. Springer, New York. \url{https://doi.org/10.1007/978-1-4419-0217-7}

\bibitem{broug2019}
Broug, E. (2019). \textit{Islamic Geometric Patterns}. Revised and expanded edition. Thames and Hudson, London. ISBN 9780500294680.

\bibitem{necipoglu1996}
Necipo\u{g}lu, G. (1996). \textit{The Topkap{\i} Scroll: Geometry and Ornament in Islamic Architecture}. Getty Publications, Santa Monica. ISBN 9780892363353.

\bibitem{abas1995}
Abas, S. J., Salman, A. S. (1995). \textit{Symmetries of Islamic Geometrical Patterns}. World Scientific, Singapore. ISBN 9789810217044.

\bibitem{lu2007}
Lu, P. J., Steinhardt, P. J. (2007). Decagonal and quasi-crystalline tilings in medieval Islamic architecture. \textit{Science}, 315(5815), 1106--1110. \url{https://doi.org/10.1126/science.1135491}

\bibitem{grunbaum1987}
Gr\"unbaum, B., Shephard, G. C. (1987). \textit{Tilings and Patterns}. W. H. Freeman, New York. ISBN 9780716711933.

\bibitem{kaplan2000}
Kaplan, C. S. (2000). Computer generated Islamic star patterns. In \textit{Proceedings of Bridges 2000: Mathematical Connections in Art, Music, and Science}, 105--112.

\bibitem{kaplan2004}
Kaplan, C. S., Salesin, D. H. (2004). Islamic star patterns in absolute geometry. \textit{ACM Transactions on Graphics}, 23(2), 97--119. \url{https://doi.org/10.1145/990002.990003}

\bibitem{kaplan2005}
Kaplan, C. S. (2005). Islamic star patterns from polygons in contact. In \textit{Proceedings of Graphics Interface 2005}, 177--185. \url{https://doi.org/10.5555/1089508.1089538}

\bibitem{ranjazmay2023}
Ranjazmay Azari, M., Bemanian, M., Mahdavinejad, M., K\"orner, A., Knippers, J. (2023). Application-based principles of Islamic geometric patterns; state-of-the-art, and future trends in computer science/technologies: a review. \textit{Heritage Science}, 11, Article 22. \url{https://doi.org/10.1186/s40494-022-00852-w}

\bibitem{sayed2016}
Sayed, Z., Ugail, H., Palmer, I., Purdy, J., Reeve, C. (2016). Auto-parameterized shape grammar for constructing Islamic geometric motif-based structures. In Sourin, A., Gavrilova, M., Tan, C. (eds), \textit{Transactions on Computational Science XXVIII}. Lecture Notes in Computer Science, vol. 9590, 146--162. Springer, Berlin, Heidelberg. \url{https://doi.org/10.1007/978-3-662-53090-0_8}

\bibitem{elmahmudi2019deep}
Elmahmudi, A., Ugail, H. (2019). Deep face recognition using imperfect facial data. \textit{Future Generation Computer Systems}, 99, 213--225. \url{https://doi.org/10.1016/j.future.2019.04.025}

\bibitem{ugail2026forensic}
Ugail, H., Alawar, H. M., Zehi, A. A., Alkendi, A. M., Jaleel, I. L. (2026). Evaluation of latent diffusion enhanced face recognition under forensic image degradations. \textit{Discover Computing}, 29, Article 193. \url{https://doi.org/10.1007/s10791-026-10082-4}

\bibitem{ugail2026sketches}
Ugail, H., Ritch-Frel, J., Matuzava, I., Stork, D. G. (2026). Verification of historical sketches via one-class learning on compact feature representations. \textit{PLOS ONE}, 21(6), e0344796. \url{https://doi.org/10.1371/journal.pone.0344796}

\bibitem{huang2023beyond}
Huang, W., Jia, X., Zhong, X., Wang, X., Jiang, K., Wang, Z. (2023). Beyond the parts: learning coarse-to-fine adaptive alignment representation for person search. \textit{ACM Transactions on Multimedia Computing, Communications, and Applications}, 19(3), Article 105. \url{https://doi.org/10.1145/3565886}

\bibitem{ha2018}
Ha, D., Eck, D. (2018). A neural representation of sketch drawings. In \textit{International Conference on Learning Representations}. arXiv:1704.03477.

\bibitem{carlier2020}
Carlier, A., Danelljan, M., Alahi, A., Timofte, R. (2020). DeepSVG: a hierarchical generative network for vector graphics animation. In \textit{Advances in Neural Information Processing Systems}, 33, 16351--16361.

\bibitem{li2020}
Li, T.-M., Luk\'{a}\v{c}, M., Gharbi, M., Ragan-Kelley, J. (2020). Differentiable vector graphics rasterization for editing and learning. \textit{ACM Transactions on Graphics}, 39(6), Article 193. \url{https://doi.org/10.1145/3414685.3417871}

\bibitem{reddy2021}
Reddy, P., Gharbi, M., Luk\'{a}\v{c}, M., Mitra, N. J. (2021). Im2Vec: synthesizing vector graphics without vector supervision. In \textit{Proceedings of the IEEE/CVF Conference on Computer Vision and Pattern Recognition}, 7342--7351. \url{https://doi.org/10.1109/CVPR46437.2021.00726}

\bibitem{ugail2026neuraltension}
Ugail, H., Howard, N. (2026). A neural tension operator for curve subdivision across constant curvature geometries. \textit{arXiv preprint} arXiv:2603.28937.

\bibitem{bronstein2017}
Bronstein, M. M., Bruna, J., LeCun, Y., Szlam, A., Vandergheynst, P. (2017). Geometric deep learning: going beyond Euclidean data. \textit{IEEE Signal Processing Magazine}, 34(4), 18--42. \url{https://doi.org/10.1109/MSP.2017.2693418}

\bibitem{han2025}
Han, J., Cen, J., Wu, L., Li, Z., Kong, X., Jiao, R., Yu, Z., Xu, T., Wu, F., Wang, Z., Xu, H., Wei, Z., Zhao, D., Liu, Y., Rong, Y., Huang, W. (2025). A survey of geometric graph neural networks: data structures, models and applications. \textit{Frontiers of Computer Science}, 19, Article 1911375. \url{https://doi.org/10.1007/s11704-025-41426-w}

\bibitem{cohen2016}
Cohen, T. S., Welling, M. (2016). Group equivariant convolutional networks. In \textit{Proceedings of the 33rd International Conference on Machine Learning}, PMLR 48, 2990--2999.

\bibitem{kondor2018}
Kondor, R., Trivedi, S. (2018). On the generalization of equivariance and convolution in neural networks to the action of compact groups. In \textit{Proceedings of the 35th International Conference on Machine Learning}, PMLR 80, 2747--2755.

\bibitem{weiler2019}
Weiler, M., Cesa, G. (2019). General E(2)-equivariant steerable CNNs. In \textit{Advances in Neural Information Processing Systems}, 32.

\bibitem{satorras2021}
Satorras, V. G., Hoogeboom, E., Welling, M. (2021). E(n) equivariant graph neural networks. In \textit{Proceedings of the 38th International Conference on Machine Learning}, PMLR 139, 9323--9332.

\bibitem{khemani2024gnnreview}
Khemani, B., Patil, S., Kotecha, K., Tanwar, S. (2024). A review of graph neural networks: concepts, architectures, techniques, challenges, datasets, applications, and future directions. \textit{Journal of Big Data}, 11, Article 18. \url{https://doi.org/10.1186/s40537-023-00876-4}

\bibitem{kipf2017}
Kipf, T. N., Welling, M. (2017). Semi-supervised classification with graph convolutional networks. In \textit{International Conference on Learning Representations}. arXiv:1609.02907.

\bibitem{hamilton2017}
Hamilton, W. L., Ying, R., Leskovec, J. (2017). Inductive representation learning on large graphs. In \textit{Advances in Neural Information Processing Systems}, 30, 1024--1034.

\bibitem{velickovic2018}
Veli\v{c}kovi\'{c}, P., Cucurull, G., Casanova, A., Romero, A., Li\`{o}, P., Bengio, Y. (2018). Graph attention networks. In \textit{International Conference on Learning Representations}. arXiv:1710.10903.

\bibitem{simonovsky2018}
Simonovsky, M., Komodakis, N. (2018). GraphVAE: towards generation of small graphs using variational autoencoders. In Kurkov\'a, V., Manolopoulos, Y., Hammer, B., Iliadis, L., Maglogiannis, I. (eds), \textit{Artificial Neural Networks and Machine Learning -- ICANN 2018}. Lecture Notes in Computer Science, vol. 11139, 412--422. Springer, Cham. \url{https://doi.org/10.1007/978-3-030-01418-6_41}

\bibitem{you2018}
You, J., Ying, R., Ren, X., Hamilton, W. L., Leskovec, J. (2018). GraphRNN: generating realistic graphs with deep auto-regressive models. In \textit{Proceedings of the 35th International Conference on Machine Learning}, PMLR 80, 5708--5717.

\bibitem{vignac2023}
Vignac, C., Krawczuk, I., Siraudin, A., Wang, B., Cevher, V., Frossard, P. (2023). DiGress: discrete denoising diffusion for graph generation. In \textit{International Conference on Learning Representations}. arXiv:2209.14734.

\bibitem{diligenti2017}
Diligenti, M., Gori, M., Sacc\`{a}, C. (2017). Semantic-based regularization for learning and inference. \textit{Artificial Intelligence}, 244, 143--165. \url{https://doi.org/10.1016/j.artint.2015.08.011}

\bibitem{nandwani2019}
Nandwani, Y., Pathak, A., Mausam, Singla, P. (2019). A primal-dual formulation for deep learning with constraints. In \textit{Advances in Neural Information Processing Systems}, 32, 12157--12168.

\bibitem{li2023kgshfp}
Li, Z., Zhang, Q., Zhu, F., Li, D., Zheng, C., Zhang, Y. (2023). Knowledge graph representation learning with simplifying hierarchical feature propagation. \textit{Information Processing \& Management}, 60(4), Article 103348. \url{https://doi.org/10.1016/j.ipm.2023.103348}

\end{thebibliography}
\end{document}